\documentclass[twocolumn]{article}

\usepackage{arxiv}

\usepackage{times}
\usepackage{epsfig}
\usepackage{graphicx}
\usepackage{amsmath}
\usepackage{amssymb}
\usepackage[toc, page]{appendix}

\usepackage{booktabs}
\usepackage{subcaption}
\usepackage{amsthm}

\newtheorem{theorem}{Theorem}

\newcommand\independent{\protect\mathpalette{\protect\independenT}{\perp}}
\def\independenT#1#2{\mathrel{\rlap{$#1#2$}\mkern2mu{#1#2}}}

\usepackage[pagebackref=true,breaklinks=true,colorlinks,bookmarks=false]{hyperref}

\newcommand{\Image}{I}

\newcommand{\Lab}{L}
\newcommand{\lab}{l}
\newcommand{\Labset}{\mathcal L}
\newcommand{\Signal}{S}

\newcommand{\confounder}{b}
\newcommand{\Confounder}{B}
\newcommand{\Confounderset}{\mathcal B}
\newcommand{\ConfounderTrue}{B^{\ast}}
\newcommand{\Representation}{R}

\newcommand{\Prediction}{R}
\newcommand{\prediction}{r}
\newcommand{\Predictionset}{\mathcal R}

\begin{document}

\title{Towards Learning an Unbiased Classifier from Biased Data via Coditional
    Adversarial Debiasing}

\author{Christian Reimers${}^{0, 1, 2}$,\quad Paul Bodesheim$^1$, \quad
    Jakob Runge$^{2,3}$ \quad and \quad Joachim Denzler$^{1,2}$\\ 
    \begin{tabular}{ccc} 
        $^1$Computer Vision Group, & $^2$Institute of Data Science, & \\
        Friedrich Schiller University Jena, & German Areospace Center (DLR), & $^3$Technische Universit\"at Berlin \\
        07743 Jena, Gremany & 07745 Jena, Germany & 10623 Berlin, Germany
    \end{tabular}
}

\twocolumn[
    \begin{@twocolumnfalse}
        \maketitle

        \begin{abstract}
            Bias in classifiers is a severe issue of modern deep learning methods, 
            especially for their application in safety- and security-critical areas. 
            Often, the bias of a classifier is a direct consequence of a bias in the 
            training dataset, frequently caused by the co-occurrence of relevant 
            features and irrelevant ones. To mitigate this issue, we require learning 
            algorithms that prevent the propagation of bias from the dataset into the 
            classifier. We present a novel adversarial debiasing method, which addresses 
            a feature that is spuriously connected to the 
            labels of training images but statistically independent of the 
            labels for test images. Thus, the automatic identification of relevant 
            features during training is perturbed by irrelevant  features. 
            This is the case in a wide range of bias-related problems for many 
            computer vision tasks, such as automatic skin cancer detection or driver 
            assistance.
            We argue by a mathematical proof that our approach is superior to existing 
            techniques for the abovementioned bias. 
            Our experiments show that our approach performs better than 
            state-of-the-art techniques on a well-known benchmark dataset with 
            real-world images of cats and dogs.
        \end{abstract}
        \vspace{4em}
    \end{@twocolumnfalse}
]

\footnotetext[0]{Corresponding Author (christian.reimers@uni-jena.de)}
\section{Introduction}
    Deep neural networks have demonstrated impressive performances in many areas. 
    These areas encompass not only classical computer vision tasks, like object 
    detection or semantic segmentation, but also safety- and security-critical 
    tasks, such as skin cancer detection \cite{perez2018data} or predicting 
    recidivism \cite{angwin2016machine}. However, many people, including domain 
    experts, advise against employing deep learning in those 
    applications, even if these classifiers outperform human experts, for 
    example, in skin lesion classification \cite{tschandl2019comparison}. 
    One reason for their concerns is bias in the classifiers. 
    Indeed, almost all image datasets contain some kind of bias 
    \cite{wang2020revise} and, consequently, the 
    performance of classifiers varies significantly across subgroups. 
    For example, the skin lesion classification performance varies
    across age groups \cite{muckatira2020properties}, and recidivism prediction 
    is biased against ethnic groups \cite{angwin2016machine}.
    
    \begin{figure}
        \includegraphics[width=\linewidth]{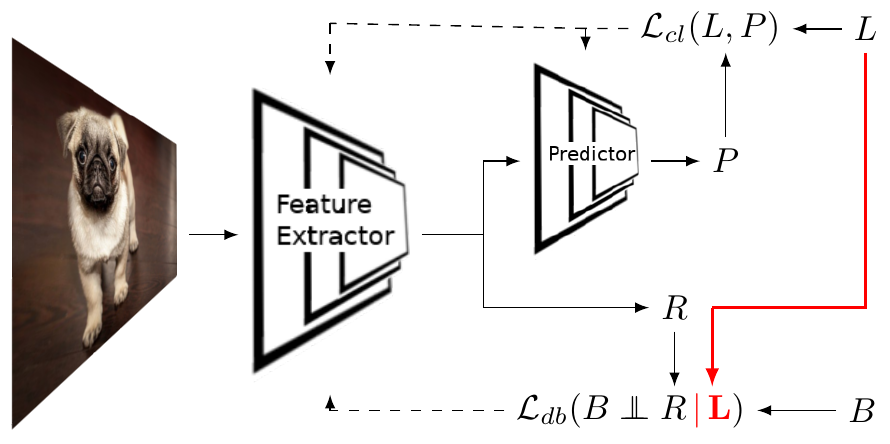}
        \caption{
            In adversarial debiasing, a debiasing loss $\mathcal L_{db}$ is 
            often used to enforce independence between the bias variable 
            $B$ and a representation $R$. In this work, we show that it is 
            beneficial to condition this independence on the label $L$.}  
        \label{fig:teaser}
    \end{figure}

    One major reason for bias in classifiers is dataset bias. Every dataset is a
    unique slice through the visual world \cite{torralba2011unbiased}. 
    Therefore, an image dataset often does not represent the real world 
    perfectly but contains spurious dependencies between meaningless 
    features and the labels of its samples. This spurious connection can be 
    caused by incautious data collection or by 
    justified concerns. If, for example, the acquirement of particular examples 
    is dangerous, these examples might be left out of a dataset due to justified 
    safety concerns.
    A classifier trained on such a dataset might pick the spuriously
    dependent feature to predict the label and is, thus, biased.
    In order to mitigate such a bias, it is important to understand the nature of 
    the spurious dependence. Therefore, we start our investigation at the 
    data generation process. We provide a formal description of the data
    generation model for a common computer vision bias in Section~\ref{sec:bias}. 
    In contrast to other approaches that do not provide a model for the data 
    generation process and, hence, rely solely on empirical evaluations, this 
    allows us to investigate our proposed method theoretically. We discuss 
    the resulting differences to related work in Section~\ref{sec:related}.
    Additionally, our model provides a simple way for practitioners to determine 
    whether our solution applies to a specific problem.

    The main contribution of our work is a novel adversarial debiasing strategy. 
    The basic concept of adversarial 
    debiasing and the idea of our improvement can be observed in 
    Figure~\ref{fig:teaser}. For adversarial debiasing, a second loss $\mathcal L_{db}$ is used in 
    addition to the regular training loss $\mathcal L_{cl}$ of a neural network classifier.
    This second loss penalizes the dependence between the bias variable 
    $\Confounder$  and an intermediate representation $\Representation$ from the 
    neural network. The main difference we propose in this paper is replacing 
    this dependence $\Confounder \not\independent \Representation$ by the 
    conditional dependence $\Confounder \not\independent \Representation \,|\, 
    \Lab$ with $\Lab$ being the label. In fact, it turns out that this conditional dependence is better 
    suited than the unconditional dependence for the considered kind of bias. 
    The motivation for this replacement can be found in 
    Section~\ref{sec:theory}. 
    Even more important, the formal description of the data generation model 
    allows us to provide a rigorous mathematical proof for the linear case in 
    Section~\ref{sec:theory}, which can also be extended to the non-linear case.
    This proof demonstrates that our new approach fits the specific bias well. 
    
    To use our new conditional independence criterion for adversarial debiasing,
    we have to implement it as a differentiable loss. We provide three possible
    implementations in Section~\ref{sec:implementation}. To this end, we extend
    existing ideas from 
    \cite{perez2017fair, kim2019learning, Li2019kernel, adeli2019bias} 
    for implementing the unconditional independence criterion and provide 
    realizations for their conditional counterparts. We demonstrate that these 
    new loss functions lead to larger accuracies on unbiased test sets. 
    In Section~\ref{sec:synthetic}, we provide results of experiments on a 
    synthetic dataset that maximizes the difference between the conditional and 
    the unconditional dependence. Further, in Section~\ref{sec:real}, we present 
    results on a dataset with real-world images of cats and dogs that is used by 
    previous work to evaluate adversarial debiasing. These 
    experiments show that our new approach outperforms existing methods on both 
    synthetic and real-world data. 
    Further experiments shown in Section~\ref{sec:ablation} indicate that the proposed change of the criterion 
    causes the increasing accuracies. 


\section{Related work}\label{sec:related}
    The goal of debiasing is to prevent a classifier from using biased features. 
    To reach this goal, we first have to choose a criterion to determine whether 
    the classifier uses a feature. Second, we have to turn this criterion into a
    differentiable loss. In this section, we compare our choices to related work 
    from the literature.

    First, we compare our criterion for determining whether the classifier uses 
    a feature. Traditionally, adversarial debiasing aims to learn a feature 
    representation that is informative for a task but independent of 
    the bias. 
    Hence, a second neural network that should predict the bias from the feature 
    representation is introduced to enforce this independence. The 
    original network for classification and this second network are then trained 
    in an adversarial fashion. 
    To this end, different loss functions for the 
    original network are suggested to decrease the performance of the second 
    network for predicting the bias. 
    In 
    \cite{alvi2018turning}, the authors 
    minimize the cross-entropy between bias prediction and a uniform 
    distribution. The mean squared error between the reconstruction and the bias 
    is used in \cite{zhang2018mitigating}. The authors of \cite{kim2019learning} 
    maximize the cross-entropy between the predicted distribution of the bias 
    and the bias variable. Additionally, they maximize the entropy of the 
    distribution of the predicted bias. In \cite{adeli2019bias}, the authors 
    minimize the correlation between the ground-truth bias and the prediction of 
    the bias. However, as demonstrated in \cite{reimers2020determining}, 
    independence is too restrictive as a criterion for determining whether a 
    deep neural network uses a certain feature. This fact is also reflected in 
    the experimental results of the abovementioned papers. The resulting 
    classifiers are less biased, but this often leads to decreasing performance 
    on unbiased test sets. For example, \cite{alvi2018turning} report 
    significantly less bias in an age classifier trained on a dataset biased by 
    gender but the performance on the unbiased test set drops from $0.789$ to 
    $0.781$. Our work is fundamentally different. Instead of a different loss, 
    we suggest a different criterion to determine whether a neural network uses 
    a feature. We use the conditional independence criterion proposed by
    \cite{reimers2020determining} rather than independence between the
    representation and the bias.

    To turn the chosen criterion, in our case, conditional independence, into
    a differentiable loss, we extend three ideas from the literature. We build 
    on work of \cite{perez2017fair} and \cite{Li2019kernel}, which use the
    Hilbert-Schmidt independence criterion (HSIC) \cite{gretton2008kernel} as well as
    on the ideas of using mutual information presented in \cite{kim2019learning} or the 
    predictability criterion presented by \cite{adeli2019bias}. All three 
    criteria are unconditional. Our work extends them to conditional 
    independence criteria.

    For understanding deep neural networks, \cite{reimers2020determining} 
    demonstrate that conditional dependence is a sharper criterion than 
    unconditional dependence. For adversarial domain adaptation, 
    \cite{wang2020classes} show significant improvements of conditional 
    adversarial losses compared to unconditional adversarial losses. However, 
    our work is the first one that makes use of these advantages in adversarial 
    debiasing.

    The authors of \cite{rieger2020interpretations} use contextual decomposition 
    \cite{murdoch2018beyond} to force a neural network to focus on useful
    areas of an image. In contrast, we use the approach of 
    \cite{reimers2020determining}, which can not only be applied to image areas 
    but arbitrary features. 

    While the vast majority of adversarial debiasing methods acknowledge that
    bias has many forms, they rarely link the suggested solutions to the 
    processes that generate the biased data. Instead, they rely exclusively on 
    empirical evaluations. In contrast, we provide a specific model for a 
    specific kind of bias as well as a theoretical proof that our approach is 
    better suited for this case.

\section{Proposed debiasing approach}\label{sec:method}
    In this section, we introduce and motivate our novel approach to adversarial
    debiasing. First, in Section~\ref{sec:bias}, we define a model for the bias 
    we consider in this work. Even though the bias model is quite specific, it 
    covers many relevant cases in computer vision. We corroborate this claim 
    with two examples at the end of Section~\ref{sec:bias}. Afterward, we 
    introduce our novel adversarial debiasing criterion. 
    Section~\ref{sec:theory}  provides a theoretical motivation and a mathematical 
    proof that the new criterion fits our specific bias model better than 
    existing solutions from the literature.
    Finally, we provide three possible implementations for loss functions that realize this 
    criterion in Section~\ref{sec:implementation}.

    \subsection{Bias model}\label{sec:bias}
        \begin{figure}[t]
            \centering
            \includegraphics{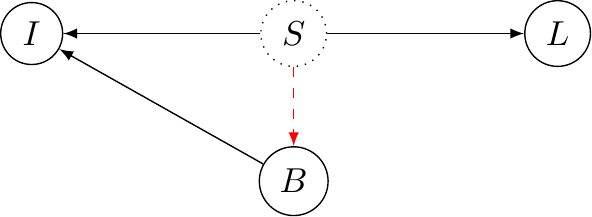}
            \caption{A graphical representation of the specific bias. Circles 
            represent variables,  dotted circles represent unobserved variables. 
            The label $\Lab$ is only dependent on a signal $\Signal$, while the
            input $\Image$ is also dependent on some 
            variable $\Confounder$. In the training set, the signal $\Signal$ 
            influences the variable $\Confounder$ due to bias. This is indicated by the 
            red dashed arrow.}
            \label{fig:bias}
        \end{figure}
        Many different kinds of bias exist and influence visual datasets in 
        various ways \cite{wang2020revise}. In this work, we consider a specific 
        kind of bias. We will later argue that this specific model covers many 
        relevant tasks in computer vision.


        To describe the bias model, we start with a graphical model of the 
        underlying data generation process displayed in Figure~\ref{fig:bias}.
        For classification tasks, like separating cats 
        from dogs, we assume that a process following this graphical model 
        generates the label $\Lab$ (cat or dog) from a signal $\Signal$. This 
        signal $\Signal$ is contained in and can be extracted from the input
        $\Image$. 
        However, the input $\Image$ is a mixture of multiple signals: Besides 
        $\Signal$, another signal $\Confounder$ influences $\Image$. In the 
        cat/dog example, $\Confounder$ might relate to the fur’s color. 
        Since the fur’s color is not meaningful in distinguishing cats and dogs, 
        $\Confounder$ is independent of $\Signal$ and $\Lab$ 
        during the application of the machine learning method in practice, i.e.,
        on an unbiased test set, 
        \begin{equation}
            \textup{Test:} \quad \Confounder \independent \Lab .
        \end{equation}
        In contrast, in a biased training set, we find an unwanted dependence 
        between the signal $\Signal$ and the signal
        $\Confounder$. Hence, we call $\Confounder$ the bias variable in this case.
        The training dataset in the cat/dog example might contain only images of 
        dogs with bright fur and images of cats with dark fur, leading to
        \begin{equation} 
            \textup{Training:} \quad \Confounder \not\independent \Lab .
        \end{equation}
        This dependence can be utilized by a machine-learning algorithm to 
        predict $\Lab$ using $\Confounder$, resulting in a biased classifier. 

        To better understand the direction of the arrow from $\Signal$ to 
        $\Confounder$, we want to emphasize, that data for a task is selected 
        with a purpose. Images are included in the dataset because they show cats or dogs and 
        one will, if necessary, deliberately accept imbalances in variables like 
        fur-color. In contrast, if one find that our dataset misrepresents 
        fur-color one would never accept a major misrepresentation of the ratio 
        of cats and dogs to compensate for this problem. This demonstrates, that
        $\Signal$ influences $\Confounder$ through the dataset creation while 
        $\Confounder$ does not influence $\Signal$.


        This bias model covers many relevant situations in computer vision. In 
        the following, we give two examples where the bias model fits the data 
        and one where it does not. 
        
        The first example is a driver assistance system that uses a camera to 
        estimate aquaplaning risk \cite{hartmann2018aquaplaning}. To train such 
        a system, a dataset is needed that contains images of safe conditions 
        and aquaplaning conditions. While the images of safe conditions can be 
        easily collected in the wild, it is dangerous to drive a car under 
        aquaplaning conditions. Therefore, the images of aquaplaning conditions 
        must be collected in a specific facility. In this example, the signal 
        $\Signal$ is the standing water, and the bias variable $\Confounder$ is 
        the location that determines the background of the image. Because of 
        this safety risk, they are dependent in the training set, but not at the 
        time of application. 
        
        The second example is automatic classification of skin 
        lesions from images \cite{tschandl2018ham10000}. The classification 
        systems are trained on images taken by dermatologists. Since the growth 
        of the skin lesion is informative for skin lesion classification, 
        dermatologists sometimes draw a scale next to the skin lesion if they 
        suspect it is malignant. In this example, the characteristics of the 
        skin lesion form the signal $\Signal$, while the drawn scale is the bias 
        variable $\Confounder$. These are dependent in the training set, but 
        this bias is not present in the application. 
        
        The third example is the one, where the bias model does not fit the 
        data. The example is a system that predicts absenteeism 
        in the workplace, for example \cite{ali2020enhanced}. If an automated 
        system predicts absenteeism, it might be unfair to women because of 
        pregnancy. And we might want a system that does not take this effect 
        into account. Here, the bias variable $\Confounder$ is the sex and the 
        signal $\Signal$ is the time an employee will be absent from work. Here, 
        our bias model does not fit because the data contain a $\Confounder$ to 
        $\Lab$ link.

    \subsection{Conditional independence for debiasing}\label{sec:theory}
        Deep neural networks unite a feature extractor and a predictor
        \cite{reimers2020deep}. For adversarial debiasing, we separate the two 
        at some intermediate layer. We denote the output of the feature 
        extractor $\Representation$. 
        Note that it is a valid approach to use the whole network for 
        feature extraction. In this case, $\Representation$ is the 
        prediction of the neural network. Both networks are 
        trained using a classification loss $\mathcal L_{cl}$, e.g., cross-entropy loss. 
        Additionally, a debiasing loss $\mathcal L_{db}$ is used to prevent the extraction of the 
        bias variable $\Confounder$. 
        For a visualization, see Figure~\ref{fig:teaser}. 
        Most approaches for adversarial debiasing
        \cite{alvi2018turning, zhang2018mitigating, adeli2019bias, 
        kim2019learning} aim to find a representation $\Representation$ of $\Image$
        that is independent of the bias variable $\Confounder$ while still 
        being informative for the label $\Lab$, i.e.,
        \begin{equation}\label{eqn:ad_old}
            \Representation \independent \Confounder \quad\land\quad  \Representation 
                \not\independent \Lab.
        \end{equation}
        In this work, we propose a novel strategy: Instead of 
        independence, we aim for conditional independence of $\Representation$ 
        and $\Confounder$, given the label $\Lab$, i.e.,
        \begin{equation}\label{eqn:ad_new}
            \Representation \independent \Confounder \, |\, \Lab \quad\land\quad 
                \Representation \not\independent \Lab.
        \end{equation}
        Our strategy is better suited for the specific bias model presented in 
        Section~\ref{sec:bias}. In this section, we show that our strategy 
        agrees with state-of-the-art results in explaining deep neural networks 
        \cite{reimers2020determining}. An optimal classifier fulfills the 
        conditional independence \eqref{eqn:ad_new} but not the independence 
        \eqref{eqn:ad_old}. We prove this statement for the case that all 
        data generation processes are linear. 
        Consequently, loss functions that enforce the independence
        \eqref{eqn:ad_old} will decrease the classifier's performance, 
        while loss functions that ensure the conditional independence 
        \eqref{eqn:ad_new} will not. 

        The goal of debiasing is to prevent a deep neural network from using a 
        biased feature. To reach this goal, we first need to determine whether a 
        classifier uses a feature. 
        So far, most approaches for adversarial debiasing use the dependence 
        between a feature and the classifier's prediction to measure whether a 
        classifier is using a feature. In contrast, we build on previous work 
        for understanding deep neural networks \cite{reimers2020determining}.
        While the independence criterion \eqref{eqn:ad_old} obviously ensures 
        that a bias variable $\Confounder$ is not used for
        classification, the authors of \cite{reimers2020determining} reveal that 
        independence is too restrictive to determine whether a deep neural 
        network uses a certain feature. They employ the framework 
        of causal inference \cite{pearl2009causality} to show that the 
        ground-truth labels are a confounding variable for features of the input 
        and the predictions of a deep neural network. In theoretical 
        considerations and empirical experiments, they further demonstrate that 
        the prediction of a neural network and a feature of the input can be 
        dependent even though the feature is not used by the deep neural 
        network. The authors, therefore, suggest using the conditional 
        independence \eqref{eqn:ad_new}, which we employ in our method for 
        adversarial debiasing.

        Thus, the independence criterion \eqref{eqn:ad_old} 
        is too strict. Even if the deep neural network ignores the bias, it 
        might not satisfy \eqref{eqn:ad_old} and, hence, not minimize a 
        corresponding loss. Furthermore, minimizing such a loss based on the 
        independence criterion will likely result in a less accurate classifier. 
        To corroborate this claim, we present a mathematical proof for the 
        following statement. If the bias can be modeled as explained in 
        Section~\ref{sec:bias}, the optimal classifier, which recovers the 
        signal and calculates the correct 
        label for every input image, fulfills the conditional independence 
        \eqref{eqn:ad_new} but not the independence \eqref{eqn:ad_old}. 
        In this work, we only include the proof for the linear case, i.e., all 
        data generating processes are linear.
        However, this proof can further be extended to the non-linear case by 
        using a kernel space in which the data generation processes are linear 
        and replacing covariances with the inner product of that space.
        
        \begin{theorem}
            If the bias can be modeled as described in Section~\ref{sec:bias},
            the optimal classifier fulfills the conditional independence in
            (\ref{eqn:ad_new}) but not the independence in (\ref{eqn:ad_old}).
        \end{theorem}
        \begin{proof}
            Throughout this proof, we denote all variables with capital Latin 
            letters. 
            Capital Greek letters denote processes.
            For these processes, we denote the 
            linear coefficients with lower-case Greek letters. 
            The only exception to this is the optimal classifier that is denoted 
            by $F^\ast$.

            We start the proof by defining all functions involved in the model.
            Afterward, since dependence results in correlation in the linear 
            case, a simple calculation proves the claim.
            Let $\Signal$ denote the signal according to the bias model, as 
            explained in Section~\ref{sec:bias}. 
            Since we are in the linear case, the bias variable $\Confounder$ can 
            be split into a part that is fully determined by $\Signal$ and a 
            part that is independent of $\Signal$.

            Let $\ConfounderTrue$ be the part of the bias
            variable that is independent of $\Signal$. The bias variable $\Confounder$
            is given by 
            \begin{equation}
                \Confounder = \alpha_1 \Signal + \alpha_2 \ConfounderTrue
                =: \Phi \left(\Signal, \ConfounderTrue\right).
            \end{equation}
            Further, the label $\Lab$ can be calculated from the signal 
            $\Signal$ 
            \begin{equation}\label{eqn:labzs}
                \Lab = \zeta_1 \Signal =: \Xi \left(\Signal\right)
            \end{equation}
            and the image $\Image$ is given by
            \begin{equation}
                \Image =: \Psi \left(\Signal, \Confounder\right) = \Psi\left(
                    \Signal,  \Phi \left(\Signal, \ConfounderTrue\right)\right).
            \end{equation}
            The optimal solution $F^\ast$ of the machine learning problem will recover the
            signal and calculate the label. By the assumptions of the bias 
            model, the signal can be recovered from the input. Thus, there
            exists a function $\Psi^{\dagger}$ such that
            \begin{equation}
                \Psi^{\dagger} \left(\Psi \left(\Signal, \Confounder\right)
                \right) = \Signal
            \end{equation}
            holds. Therefore, $F^\ast$ is given by
            \begin{equation}
                F^\ast := \Xi \Psi^{\dagger}.
            \end{equation}
            Now, we have defined all functions appearing in the model. The rest 
            of the proof are two straightforward calculations.
            In the linear case, the independence of variables is equivalent to 
            variables being uncorrelated. We denote the covariance of two 
            variables $A, B$ with $\langle A, B \rangle$.
            To prove that \eqref{eqn:ad_old} does not hold, we calculate
            \begin{equation}\label{eqn:corr}
                \begin{split}
                    \left\langle F^\ast (\Image), \Confounder \right\rangle
                    &= \left\langle \Xi \Psi^{\dagger} \Psi \left( \Signal, \Phi \left(
                        \Signal ,\ConfounderTrue \right)\right),
                        \Phi \left(\Signal,  \ConfounderTrue \right) \right\rangle
                    \\ &= \left\langle \zeta_1\Signal, \alpha_1 \Signal + \alpha_2
                        \ConfounderTrue \right\rangle
                    = \zeta_1\alpha_1 \left\langle \Signal, \Signal \right\rangle.
                \end{split}
            \end{equation}
            This is equal to zero if and only if either all inputs contain an 
            identical signal ($\left\langle \Signal, \Signal \right\rangle = 0$), 
            the dataset is unbiased ($\alpha_1 = 0$), or the label does not 
            depend on the signal ($\zeta_1 = 0$).

            For conditional independence, we can use partial correlation and obtain
            that $\left\langle F^\ast(\Image), \Confounder \right\rangle | \Lab$
            equals 
            \begin{equation}
                \left\langle F^\ast(\Image) - \frac{\left\langle F^\ast(\Image), \Lab
                        \right\rangle}{\left\langle\Lab, \Lab
                        \right\rangle}\Lab, \Confounder - \frac{\left\langle
                        \Confounder, \Lab \right\rangle}{\left\langle\Lab, \Lab 
                        \right\rangle} \Lab \right \rangle.
            \end{equation}
            We substitute $\Lab$ by \eqref{eqn:labzs} and use the properties of the
            inner product to arrive at 
            \begin{equation}
                \begin{split}
                    &\left\langle F^\ast(\Image), \Confounder \right\rangle - \frac{\left\langle
                         \zeta_1\Signal, \zeta_1 \Signal\right\rangle\left\langle
                        \zeta_1 \Signal, \Confounder \right\rangle}{\left\langle \zeta_1
                        \Signal, \zeta_1\Signal\right\rangle}
                    \\& = \zeta_1\alpha_1 \left\langle \Signal, \Signal \right\rangle -
                        \frac{\alpha_1\zeta_1^3\left\langle \Signal, \Signal
                        \right\rangle^2}{\zeta_1^2\left\langle \Signal, \Signal
                        \right\rangle} = 0.
                \end{split}
            \end{equation}
            This completes the proof for the linear case. For more detailed 
            calculations see Section~\ref{sec:ape:calc} in the Appendix.
        \end{proof}

        The optimal classifier does not minimize loss criteria based on the 
        independence \eqref{eqn:ad_old}. Further, from 
        \eqref{eqn:corr}, we see that the dependence contains 
        $\zeta_1$, which is the correlation between the signal $\Signal$ and the 
        neural network's prediction. Loss functions based on that criterion aim to 
        reduce this parameter and, hence, will negatively affect the 
        classifier's performance. We demonstrate this effect using a synthetic 
        dataset in Section~\ref{sec:experiments}.         
        In contrast, loss terms
        based on our new criterion \eqref{eqn:ad_new} are minimized by the 
        optimal classifier. Thus, corresponding loss functions do not reduce the 
        accuracy to minimize bias.

    \subsection{Implementation details}\label{sec:implementation}
        In Section~\ref{sec:theory}, we presented two reasons that indicate why 
        our new criterion \eqref{eqn:ad_new} is better suited than the old criterion 
        \eqref{eqn:ad_old} for the bias described in Section~\ref{sec:bias}. 
        In practice, we are faced with the problem of integrating our criterion 
        into the end-to-end learning framework of deep neural networks. As a 
        consequence, we provide three possibilities to realize 
        (\ref{eqn:ad_new}) as a loss function. 

        Turning an independence criterion into a loss function is not 
        straightforward. First, the result of an independence test is binary 
        and, hence, non-differentiable. Second, we need to consider 
        distributions of variables to perform an independence test. 
        However, we only see one mini-batch at a time during the training of a deep 
        neural network. Nevertheless, multiple solutions exist in the 
        unconditional case. 
        In this section, we describe three possible solutions, namely: mutual 
        information~(MI), the Hilbert-Schmidt independence criterion~(HSIC) and 
        the maximum correlation criterion~(MCC). 
        We adapt the corresponding solutions from the unconditional case and 
        extend them to conditional independence criteria.

        The first solution makes use of the mutual information of $\Prediction$ 
        and $\Confounder$ as suggested in \cite{kim2019learning},
        \begin{equation}
            \textup{MI}(\Prediction; \Confounder) = \sum_{\prediction \in \Predictionset,
                \confounder \in \Confounderset} p_{\Prediction, \Confounder}
                (\prediction, \confounder) \log\frac{p_{\Prediction, \Confounder}
                (\prediction, \confounder)}{p_{\Prediction}(\prediction)p_{\Confounder}
                (\confounder)}.
        \end{equation}
        Here, the criterion for independence is $\textup{MI}(\Prediction; 
        \Confounder) = 0$, and the mutual information is the differentiable loss 
        function. However, to evaluate this loss, we must estimate the densities
        $p_{\Prediction, \Confounder}, p_{\Prediction}$ and $p_{\Confounder}$ in every
        step, which is difficult. To mitigate this issue, the authors of
        \cite{kim2019learning} make simplifications to find a bound that
        is tractable for single examples.
        In contrast, we use conditional independence. Our criterion is 
        $\textup{MI}(\Prediction; \Confounder| \Lab) = 0$, and the loss is given 
        by the conditional mutual information $\textup{MI}(\Prediction; 
        \Confounder| \Lab)$
        \begin{equation}
                \sum_{\substack{\lab
                    \in \Labset, \confounder \in \Confounderset\\  \prediction \in
                    \Predictionset}} p_{\Prediction, \Confounder, \Lab}\left(\prediction,
                    \confounder, \lab\right) \log\frac{p_{\Lab}(\lab)p_{\Prediction,
                    \Confounder, \Lab}\left(\prediction, \confounder, \lab\right)}{
                    p_{\Prediction, \Lab}(\prediction, \lab)p_{\Confounder, \Lab)}}.
        \end{equation}
        We use kernel density estimation on the mini-batches to determine 
        the densities and employ a Gaussian kernel with a 
        variance of one-fourth of the average pairwise distance within a 
        mini-batch.
        This setting proved best in preliminary experiments on 
        reconstructing densities.

        As a second solution, we extend the Hilbert Schmidt independence criterion
        \cite{gretton2008kernel}
        \begin{equation}
            \textup{HSIC}(\Prediction, \Confounder) = \frac{1}{(m - 1)^2}
                \textup{tr}\,K_{\Prediction}HK_{\Confounder}H.
        \end{equation}
        Here, $K_{\Prediction}$ and $K_{\Confounder}$ denote the kernel matrices for
        $\Prediction$ and $\Confounder$, respectively. For the Kronecker-Delta 
        $\delta_{ij}$ and $m$ the number of examples, $H$ is given 
        by $H_{ij} = \delta_{ij} - m^{-2}$.
        The variables are independent if and only if
        $\textup{HSIC}(\Prediction, \Confounder) = 0$ holds for a sufficiently large
        kernel space. The HSIC was suggested for classical machine learning methods by
        \cite{perez2017fair} and \cite{Li2019kernel}. 
        Since we aim for conditional independence rather than independence, we use the
        conditional independence
        criterion \cite{fukumizu2008kernel}
        \begin{equation}
            \textup{tr} \, G_\Prediction S_\Lab G_\Confounder S_\Lab = 0.
        \end{equation}
        Here, for $X \in \{\Confounder,\Prediction,\Lab\}$, we use $G_X = HK_XH$ 
        and $S_\Lab = \left(\mathbb I + 1/m G_\Lab\right)^{-1}$ with the identity matrix $\mathbb I$. 
        For the relation to HSIC
        and further explanations, we refer to \cite{fukumizu2008kernel}. 
        We use the same kernel as above and estimate the loss on 
        every mini-batch independently.
        
        The third idea we extend is the predictability criterion from \cite{adeli2019bias}
        \begin{equation}\label{eqn:adeli}
            \max_{f} \quad \textup{Corr}(f(\Prediction), \Confounder) = 0.
        \end{equation}
        To use this criterion within a loss function, they parametrize $f$ by a 
        neural network. However, this is not an independence criterion as it can 
        be equal to zero, even if $\Prediction$ and $\Confounder$ are dependent. 
        Therefore, it is unclear how to incorporate the conditioning on 
        $\Lab$. As a consequence, we decided to extend the proposed criterion in 
        two ways. First, we use the maximum correlation coefficient (MCC)
       \begin{equation}\label{eqn:mcc}
           \textup{MCC}(\Prediction, \Confounder) = \max_{f,g} \quad \textup{Corr}(f(\Prediction), g(\Confounder)) = 0\quad,
       \end{equation}
        which is equal to zero if and only if the two variables are independent
        \cite{sarmanov1958maximum}. Second, we use the partial correlation 
        conditioned on the label $\Lab$, which leads to
        \begin{equation}\label{eqn:ours_mcc}
            \max_{f,g} \quad \textup{PC}(f(\Prediction), 
                g(\Confounder)\,|\,\Lab) = 0.
        \end{equation}
        To parameterize both functions $f$ and $g$, we use neural networks.
        The individual effects of the two extensions can be observed through our 
        ablation study in Section~\ref{sec:ablation}.

        Note that all three of these implementations can be used for 
        vector-valued variables. Therefore, they can also be used for multiple 
        bias variables in parallel.

\section{Experiments and results}\label{sec:experiments}
    This section contains empirical results that confirm our theoretical
    claims and form the third reason for the suitability of our proposed method.
    To this end, we first present
    experiments on a synthetic dataset that is designed to maximize the 
    difference between the independence criterion \eqref{eqn:ad_old} and the 
    conditional independence criterion \eqref{eqn:ad_new}. Afterward, we report 
    the results of an ablation study demonstrating that the gain in performance 
    can be credited to the change of the independence criterion. Finally,
    we show that our findings also apply to a real-world dataset. For this 
    purpose, we present experiments on different biased subsets of the cats and 
    dogs dataset presented in \cite{lakkaraju2016discovering}.

    To evaluate our experiments, we measure the accuracy on an unbiased testset.
    We do this for multiple reasons. First, we designed this method for 
    situations in which a dataset is biased, but we expect the system to be used
    in an unbiased, real-world situation. Hence, the accuracy on an unbiased 
    testset is our goal and evaluating it directly is the most precise measure
    for our method. 
    Nevertheless, our method is also applicable to some situations of 
    algorithmic fairness where we do not have access to an unbiased testset. In
    these situations it is common to use other evaluation methods like the 
    ``equalized odds'' \cite{hardt2016equality}
    \begin{equation}
        \Prediction \independent \Confounder \,|\, \Lab
    \end{equation} 
    or 
    ``demographic parity'' \cite{dwork2012fairness}
    \begin{equation}
        \Prediction \independent \Confounder
    \end{equation}
    both of which are the same in this situation, since the testset is unbiased. 
    The main drawback of these measures is that they are binary and, therefore,
    rather coarse-grained. Hence, we focus on the accuracy on unbiased testsets
    in this paper, but include the evaluations of these fairness criteria in 
    Section~\ref{sec:ape:fair} of the appendix.

\subsection{Synthetic data}\label{sec:synthetic}
    If a feature is independent of the label for a given classification task, the
    independence criterion \eqref{eqn:ad_old} and the conditional independence
    criterion \eqref{eqn:ad_new} agree.
    Since we aim to maximize the difference between the two criteria, we use a
    dataset with a strong dependence between the label $\Lab$ and the variable
    $\Confounder$. 
    We create a dataset of eight-by-eight pixels images that combine two 
    signals. The first signal $\Signal$, determines the
    shape of high-intensity pixels in the image. This shape is either a cross or
    a square, both consisting of the same number of pixels. The second
    signal $\Confounder$ is the color of the image. The hue of all pixels is
    either set to 0.3 (green) or to 0.9 (violet). Afterward, we add noise to the
    hue and the intensity value of every pixel. The noise is sampled from a uniform
    distribution on the interval $[-0.1, 0.1]$.
    Finally, the images are converted to the RGB colorspace. To maximize the
    dependence between the label $\Lab$ and the bias variable $\Confounder$, 
    every training image of a cross is green and every training image of a 
    square is violet. In the test set, these two signals are independent. 
    Example images from the training and test set can be seen in
    Figure~\ref{fig:examples}. We decided to limit the training set to 600
    images for two reasons. First, small datasets are the expected use-case for 
    adversarial debiasing. Second, this limitation increases the 
    difficulty of debiasing.

    \begin{figure*}[t]
        \centering
        \begin{subfigure}{0.5\textwidth}
            \includegraphics[width=\textwidth]{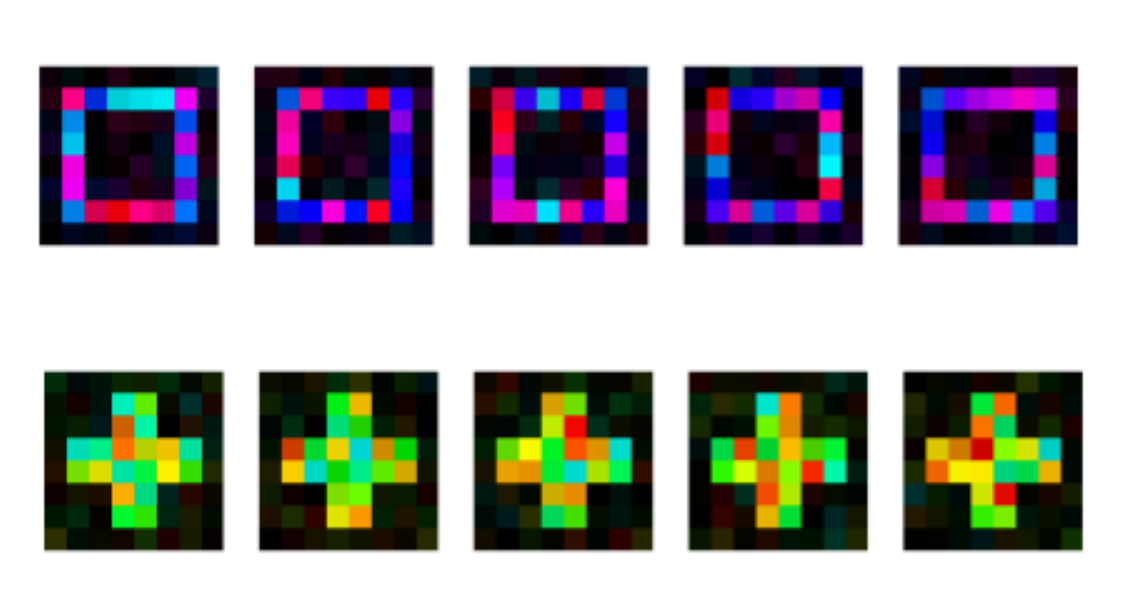}
            \subcaption{Images from the training set}
        \end{subfigure}%
        \begin{subfigure}{0.5\textwidth}
            \includegraphics[width=\textwidth]{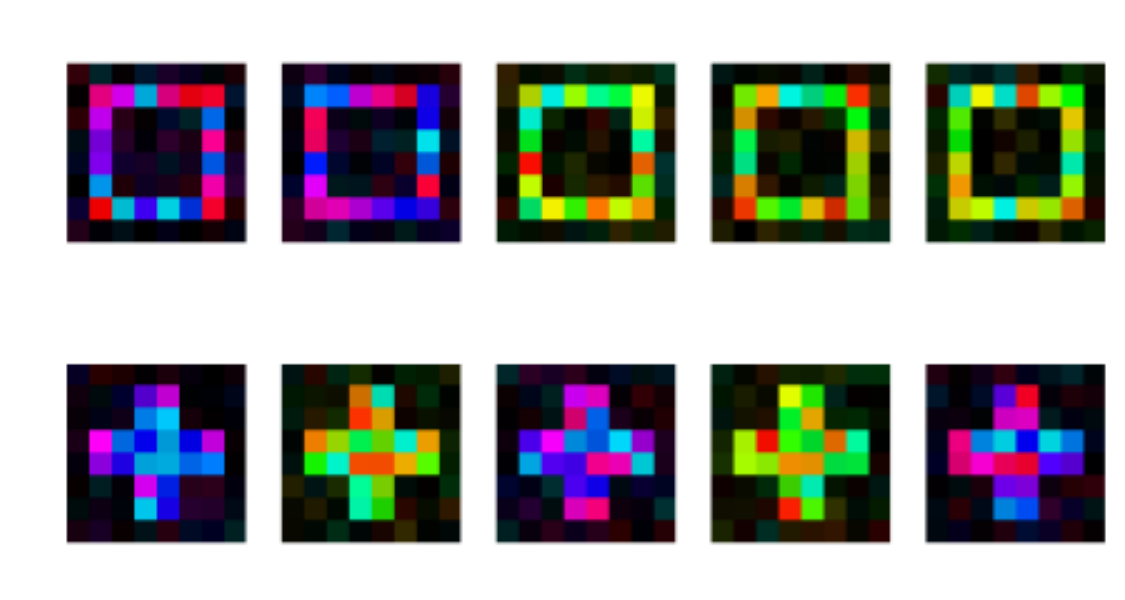}
            \subcaption{Images from the test set}
        \end{subfigure}
        \caption{Example images from the synthetic dataset. In the training set, 
        the color and the shape are dependent. In the test set, the two signals 
        are independent.}
        \label{fig:examples}
    \end{figure*}

    For our first experiment, we use the shape as a categorical label $\Lab$ and
    the color as the bias variable $\Confounder$. For this purpose, we calculate
    the mean color of the image after the noise is applied and use this value as
    the variable $\Confounder$. We report the results of this experiment in
    Table~\ref{tab:results} in the column titled Setup~I.
    To avoid any influence of shape- or color-preference, we also report results 
    for the inverse setting in a second experiment, denoted as Setup~II. For this 
    experiment, we use the color as a categorical label $\Lab$.
    The bias variable $\Confounder$ is then calculated as the difference between 
    the values of pixels in the square shape and those in the cross.

    As the backbone, we use a neural network that first contains two convolutional layers, each
    having 16 filter kernels of size $3{\times}3$. Then, a dense layer follows 
    with 128 hidden neurons. These layers use ReLU activations. 
    Finally, we use a dense layer with two 
    neurons and softmax activation for classification. 
    If a method from the literature uses an intermediate representation, we use
    the representation after the last convolutional layer. Otherwise, 
    we use the prediction of the whole network as the representation $\Prediction$.
    Whenever an additional
    neural network is required by a method, i.e., all literature methods and 
    ours based on MCC, we use a neural network with one
    hidden layer of 1024 neurons. For optimization, we use Adam 
    \cite{kingma2014adam}.
    
    As a baseline, we use the backbone neural network without any debiasing method. We have
    reimplemented four methods from the literature using the descriptions in the respective papers. 
    The corresponding references can be found in Table~\ref{tab:results}. Two methods are 
    proposed in \cite{zhang2018mitigating}. We call the first, which penalizes 
    the predictability of $\Confounder$ from $\Prediction$, 
    \emph{Zhang et al. I}. The second one, which penalizes predictability of
    $\Confounder$ from $\Prediction$ and $\Lab$, is called 
    \emph{Zhang et al. II}. 
    In addition, we report the averaged accuracies for these and all three implementations of our 
    proposed criterion in 
    Table~\ref{tab:results}. 
    
    We have set the hyperparameters for the competing methods to the values 
    found in the corresponding papers. For all unpublished or problem-specific 
    hyperparameters such as learning rates, as well as for all hyperparameters
    for our implementation, we used a grid search.
    To this end, we trained ten neural networks for each combination of 
    different hyperparameters and evaluated them on an unbiased validation set.
    We included the list of the hyperparameters for every method in 
    Section~\ref{sec:ape:hyp} in the appendix.
    For the second setup, we applied the same hyperparameter search. 
    This resulted in different hyperparameters compared to the 
    first setup for all methods, which are also described in
    Section~\ref{sec:ape:hyp} in the appendix.
    
    \begin{table}[t]
        \caption{The results from 100 runs of our method and all baseline methods.
        For both experiments, we report the mean accuracy {\scriptsize$\pm $
        standard error}. Best results are marked in \textbf{bold}}
        \label{tab:results}
        \centering
        \begin{tabular}{lcc}
            \toprule
            Method & Setup~I & Setup~II\\
            \midrule
            Baseline    & $0.819$ \scriptsize{$\pm  0.016$}
                        & $0.791$ \scriptsize{$\pm  0.016$} \\
            \midrule
            Adeli et al. \cite{adeli2019bias}       & $0.747$ \scriptsize{$\pm  0.015$}
                                                    & $0.776$ \scriptsize{$\pm  0.014$}\\
            Zhang et al. I \cite{zhang2018mitigating} & $0.736$ \scriptsize{$\pm  0.018$}
                                                    & $0.837$ \scriptsize{$\pm  0.017$} \\
            Zhang et al. II \cite{zhang2018mitigating}  & $0.747$ \scriptsize{$\pm  0.016$}
                                                        & $0.750$ \scriptsize{$\pm  0.013$} \\
            Kim et al. \cite{kim2019learning}           & $0.771$ \scriptsize{$\pm  0.012$}
                                                        & $0.767$ \scriptsize{$\pm  0.016$}\\
            \midrule
            Ours(MI) & $0.840$ \scriptsize{$\pm  0.014$} & $\textbf{0.871}$ \scriptsize{$\pm  0.012$}\\
            Ours(HSIC) & $0.846$ \scriptsize{$\pm  0.021$} & $0.868$ \scriptsize{$\pm  0.013$}\\
            Ours(MCC) & $\textbf{0.854}$ \scriptsize{$\pm  0.013$} & $0.867$ \scriptsize{$\pm 0.013$}\\
            \bottomrule
        \end{tabular}
    \end{table}

    
    We observe that the baseline reaches a test accuracy of more than 75\% in
    both experiments. Hence, hyperparameter selection influences which feature
    (color or shape) the neural network uses for classifications. Since there is 
    no analytical way to determine the best hyperparameters, this strong impact 
    might obscure the effect of debiasing methods. We tried to prevent this by 
    rigorous hyperparameter optimization and a large number of runs. 

    We see that 
    the difference between the methods from the literature using the independence 
    criterion \eqref{eqn:ad_old} and the methods using our new conditional 
    independence criterion \eqref{eqn:ad_new} is much larger than the 
    differences between methods within these groups. For the first experiment, 
    the worst method using \eqref{eqn:ad_new} performs $5.9$ percentage points 
    better than the best method using \eqref{eqn:ad_old}. The difference between 
    the best and worst methods within these groups is $1.4$ percentage points 
    and $3.5$ percentage points, respectively. In the second setup, the 
    results are similar. Only the method suggested in
    \cite{zhang2018mitigating} performs surprisingly good in this experiment.

    We draw two conclusions from these experiments. First,
    as intended, the synthetic dataset is challenging for all existing methods
    found in the literature. In the first setup, none of them were able to
    improve the results of the baseline. In the second experiment, only one 
    method outperformed the baseline.
    This coincides with observations from
    the literature that adversarial debiasing methods are challenged in
    situations with strong bias and lose accuracy to reduce bias. It
    also agrees with \eqref{eqn:corr}, which we discussed 
    at the end of Section~\ref{sec:theory}. 
    It is important to note that we were able to reach the baseline performance 
    for every method by allowing hyperparameters that deactivate the debiasing, 
    e.g., by setting the weight of the debiasing loss to zero. 
    To avoid this deactivation, we have limited the hyperparameters to the range 
    that is used in the respective publications. Second, we observe that all 
    methods that use our new debiasing criterion \eqref{eqn:ad_new} reach a 
    higher test accuracy than the baseline and, consequently, also a higher 
    accuracy than existing methods.  

\subsection{Ablation study}\label{sec:ablation}
    In the previous experiment, we observed that our new methods reached a 
    higher accuracy than the methods from the literature. To show that this 
    increase in accuracy can be attributed to the conditional independence 
    criterion \eqref{eqn:ad_new}, we conduct an ablation study for the three 
    methods described in Section~\ref{sec:implementation}.
    More specifically, we first report results for a method using unconditional 
    mutual information, calculated in the same way we calculate the conditional 
    mutual information. Second, we present results for a method using the
    unconditional HSIC as a loss. Third, we present two methods that
    investigate the gap between the method presented by 
    Adeli et al.~\cite{adeli2019bias} and our method using the conditional 
    maximal correlation coefficient. The first one uses the unconditional 
    maximum correlation coefficient, and the second one incorporates the partial 
    correlation (PC) instead of the correlation in \eqref{eqn:adeli}. We use the 
    settings and evaluation protocol of Setup~I and Setup~II from the previous 
    section. The results are presented in Table~\ref{tab:ablation}.

    \begin{table}[t]
        \caption{The results of the ablation study. Every method is trained on a
            biased training set and evaluated on an unbiased test set. We report
            the accuracy averaged over 100 runs and the standard error. 
            Best results are marked in \textbf{bold}}
        \label{tab:ablation}
        \centering
        \begin{tabular}{lcc}
            \toprule
            Method & Setup I & Setup II\\
            \midrule
            Unconditional MI & $0.583$ \scriptsize{$\pm 0.010$} & $0.833$\scriptsize{$\pm0.011$}\\
            Conditional MI & $\textbf{0.840}$ \scriptsize{$\pm  0.014$} &  $\textbf{0.871}$\scriptsize{$\pm0.012$}\\
            \midrule
            Unconditional HSIC & $0.744$ \scriptsize{$\pm 0.011$} & $0.590$\scriptsize{$\pm0.011$}\\
            Conditional HSIC & $\textbf{0.846}$ \scriptsize{$\pm  0.021$} & $\textbf{0.868}$\scriptsize{$\pm0.013$}\\
            \midrule
            Adeli et al. \cite{adeli2019bias} & $0.747$ \scriptsize{$\pm  0.015$} & $0.776$\scriptsize{$\pm0.014$} \\
            Ours(MCC) -- {\small only MCC} & $0.757$ \scriptsize{$\pm 0.016$} & $0.807$\scriptsize{$\pm0.015$}\\
            Ours(MCC) -- {\small only PC} & $0.836$ \scriptsize{$\pm 0.014$} & $0.830$\scriptsize{$\pm0.015$}\\
            Ours(MCC) -- {\small complete} & $\textbf{0.854}$ \scriptsize{$\pm  0.013$} & $\textbf{0.867}$\scriptsize{$\pm0.013$}\\
            \bottomrule
        \end{tabular}
    \end{table}

    We find that the unconditional versions of our methods, especially for
    mutual information (``Unconditional MI'') and HSIC (``Unconditional HSIC''), 
    perform worse than almost all methods from the literature (see Setup~I in
    Table~\ref{tab:results}). For the third method, we observe
    that the change from the predictability criterion 
    to the maximum correlation coefficient 
    increases the accuracy by 1.0 and
    1.8 percentage points for the correlation and partial correlation case,
    respectively. In contrast, the change from correlation to partial correlation 
    increases the accuracy by 8.9 percentage points with the
    predictability criterion and by 9.7 percentage points with the maximum
    correlation coefficient. These observations indicate that the improvements 
    that we found in Section~\ref{sec:synthetic} can be
    attributed to the difference between \eqref{eqn:ad_old} and
    \eqref{eqn:ad_new} and not to implementation details. 
    

\subsection{Real-world data}\label{sec:real}

\begin{table*}[th]
    \caption{Experimental results on the cats and dogs dataset. All methods were trained
    on a dataset in which p\% of all dogs are dark-furred dogs and $p$\% of all cats are light-furred.
    The first column of the table indicates the fraction $p$. The following columns
    contain the accuracies on an unbiased test set averaged over three runs. 
    We also report the standard error}
    \label{tab:real}
    \centering
    \begin{tabular}{cccccc}
        \toprule
        Fraction & Baseline& Adeli et al. \cite{adeli2019bias}  & 
            Zhang et al. I \cite{zhang2018mitigating} & Zhang et al. II \cite{zhang2018mitigating} & Ours(HSIC) \\
        \midrule
        0\%             &\textbf{0.627}    \tiny{$\pm 0.004$}
                        & 0.597            \tiny{$\pm 0.004$}
                        & 0.590            \tiny{$\pm 0.002$}
                        & 0.617            \tiny{$\pm 0.001$}
                        & 0.615            \tiny{$\pm 0.005$}
                        \\
        10\%            & 0.800             \tiny{$\pm 0.001$}
                        & 0.774             \tiny{$\pm 0.002$}
                        & 0.779             \tiny{$\pm 0.005$}
                        & 0.785             \tiny{$\pm 0.007$}
                        & \textbf{0.801}    \tiny{$\pm 0.001$}
                        \\
        20\%            & 0.845             \tiny{$\pm 0.003$}
                        & 0.829             \tiny{$\pm 0.000$}
                        & 0.812             \tiny{$\pm 0.002$}
                        & 0.809             \tiny{$\pm 0.005$}
                        & \textbf{0.855}    \tiny{$\pm 0.004$}
                        \\
        30\%            & 0.852             \tiny{$\pm 0.007$}
                        & 0.842             \tiny{$\pm 0.003$}
                        & 0.837             \tiny{$\pm 0.004$}
                        & 0.834             \tiny{$\pm 0.003$}
                        & \textbf{0.863}    \tiny{$\pm 0.002$}
                        \\
        40\%            & 0.859             \tiny{$\pm 0.007$}
                        & 0.855             \tiny{$\pm 0.004$}
                        & 0.870             \tiny{$\pm 0.002$}
                        & 0.850             \tiny{$\pm 0.001$}
                        & \textbf{0.875}    \tiny{$\pm 0.003$}
                        \\
        50\%            & 0.859             \tiny{$\pm 0.006$}
                        & \textbf{0.866}    \tiny{$\pm 0.003$}
                        & 0.856             \tiny{$\pm 0.001$}
                        & 0.853             \tiny{$\pm 0.001$}
                        & 0.860             \tiny{$\pm 0.002$}
                        \\
        60\%            & \textbf{0.866}    \tiny{$\pm 0.006$}
                        & 0.837             \tiny{$\pm 0.001$}
                        & 0.850             \tiny{$\pm 0.003$}
                        & 0.860             \tiny{$\pm 0.004$}
                        & 0.856             \tiny{$\pm 0.005$}
                        \\
        70\%            & 0.844             \tiny{$\pm 0.003$}
                        & 0.854             \tiny{$\pm 0.003$}
                        & 0.835             \tiny{$\pm 0.005$}
                        & 0.841             \tiny{$\pm 0.005$}
                        & \textbf{0.859}    \tiny{$\pm 0.000$}
                        \\
        80\%            & 0.829             \tiny{$\pm 0.002$}
                        & 0.822             \tiny{$\pm 0.005$}
                        & 0.820             \tiny{$\pm 0.005$}
                        & 0.826             \tiny{$\pm 0.003$}
                        & \textbf{0.836}    \tiny{$\pm 0.002$}
                        \\
        90\%            & 0.773             \tiny{$\pm 0.010$}
                        & 0.743             \tiny{$\pm 0.001$}
                        & 0.758             \tiny{$\pm 0.001$}
                        & 0.731             \tiny{$\pm 0.002$}
                        & \textbf{0.791}    \tiny{$\pm 0.004$}
                        \\
        100\%           & 0.612             \tiny{$\pm 0.001$}
                        & 0.612             \tiny{$\pm 0.004$}
                        & 0.604             \tiny{$\pm 0.001$}
                        & 0.609             \tiny{$\pm 0.001$}
                        & \textbf{0.616}    \tiny{$\pm 0.002$}
                        \\
        \bottomrule
    \end{tabular}
\end{table*}

\begin{table}[th]
    \caption{The results of the ablation study on the real data. Our conditional
    HSIC and an unconditional HSIC methods were trained
    on a dataset in which p\% of all dogs are dark-furred dogs and $p$\% of all cats are light-furred.
    The first column of the table indicates the fraction $p$. The following columns
    contain the accuracies on an unbiased test set averaged over three runs. 
    We also report the standard error}
    \label{tab:real:abl}
    \centering
    \begin{tabular}{lccc}
        \toprule
            & Unconditional & & \\
        Fraction & HSIC & HSIC (Ours) & Diff \\
        \midrule
          0\% & $0.611$\scriptsize{$\pm 0.003$} & 0.615\scriptsize{$\pm 0.005$} &  0.004 \\
         10\% & $0.759$\scriptsize{$\pm 0.012$} & 0.801\scriptsize{$\pm 0.001$} &  0.042 \\
         20\% & $0.816$\scriptsize{$\pm 0.002$} & 0.855\scriptsize{$\pm 0.004$} &  0.039 \\
         30\% & $0.834$\scriptsize{$\pm 0.002$} & 0.863\scriptsize{$\pm 0.002$} &  0.029 \\
         40\% & $0.861$\scriptsize{$\pm 0.003$} & 0.875\scriptsize{$\pm 0.003$} &  0.014 \\
         50\% & $0.863$\scriptsize{$\pm 0.004$} & 0.860\scriptsize{$\pm 0.002$} & -0.003 \\
         60\% & $0.844$\scriptsize{$\pm 0.001$} & 0.856\scriptsize{$\pm 0.005$} &  0.012 \\
         70\% & $0.835$\scriptsize{$\pm 0.003$} & 0.859\scriptsize{$\pm 0.000$} &  0.024 \\
         80\% & $0.820$\scriptsize{$\pm 0.007$} & 0.836\scriptsize{$\pm 0.002$} &  0.016 \\
         90\% & $0.757$\scriptsize{$\pm 0.003$} & 0.791\scriptsize{$\pm 0.004$} &  0.034 \\
        100\% & $0.606$\scriptsize{$\pm 0.002$} & 0.616\scriptsize{$\pm 0.002$} &  0.010 \\
        \bottomrule
    \end{tabular}
\end{table}

    After we have demonstrated the effectiveness of our approach on a synthetic
    dataset, we conduct a third experiment to investigate whether the 
    increase in accuracy can be observed in real-world image data as well. Even 
    though the described bias is present in many computer vision tasks, most datasets
    are inadequate for our investigations. To evaluate the performance of our 
    debiasing method, we require an unbiased test set. However, most datasets
    contain the same bias in their training and test sets since both sets are 
    collected in the same way. To mitigate this problem, we consider a 
    dataset, which has labels for multiple signals per image. This allows us to 
    introduce a bias in the training set but not in the test set. To this end, we   
    use the cats and dogs dataset that was used for the same 
    purpose in \cite{lakkaraju2016discovering}.
    
    This dataset contains images of cats and dogs that are
    additionally labeled as dark-furred or light-furred. We first remove 20\% of
    each class/fur combination as an unbiased test set. Then, we
    create eleven training sets with different levels of bias. We start with a
    training set that contains only light-furred dogs and only dark-furred cats.
    For each of the other sets, we increase the fraction of dark-furred dogs and
    light-furred cats by ten percent. Therefore, the last dataset contains only
    dark-furred dogs and light-furred cats. We equalize the size of all these
    training sets to perform a fair evaluation.
    Unfortunately, this limits the training set size to 80\% of the rarest
    class/fur combination. Therefore, the training sets contain only $2{,}469$
    images, which is $14.7\%$ of the original training set.

    As a backbone network, we use a ResNet-18 \cite{he2016deep}. We train the
    network for 150 epochs using the Adam optimizer \cite{kingma2014adam}. The
    learning rate follows a cosine decay with warm restarts
    \cite{loshchilov2016sgdr}. Additionally, we use random cropping during training and
    center cropping during inference \cite{simonyan2014very} as well as a
    progressive resizing scheme. 
    Whenever a method requires an additional neural
    network, we use a network with 
    one hidden layer of 1024 neurons.
    We apply a grid search to optimize hyperparameters for the baseline and
    adapt the hyperparameters of Setup~I from Section~\ref{sec:synthetic} for
    method-specific settings. We use the same implementations as in Section~\ref{sec:synthetic} 
    for the competing approaches and our implementation based on HSIC because it was
    most robust against different hyperparameters among our methods.


    The results reported in Table~\ref{tab:real} are averaged across three runs.
    Since the labels $\Lab$ and the bias variable $\Confounder$ are binary, the
    two signals are indistinguishable for 0\% and 100\% dark-furred dogs, respectively. 
    Furthermore, we obtain an unbiased training set for 50\% dark-furred dogs.
    Our method reaches the highest accuracy in seven out of the remaining eight biased
    scenarios and the highest overall accuracy of 0.875 for 40\% dark-furred
    dogs in the training set. For six out of these seven scenarios, the baseline was outside of our
    method's 95\% confidence interval. We observe that the methods from the
    literature only outperform the baseline in situations with little bias.
    This result supports our finding that these methods are not suited for
    the bias model described in Section~\ref{sec:bias}. 
   
    To further investigate the effectiveness of our method we compare the 
    conditional and unconditional HSIC in this setting as well. The results are
    reported in Table~\ref{tab:real:abl}. They are averaged over three runs.
    We find that the conditional HSIC outperforms the unconditional HSIC in all
    biased scenarios. The stronger the bias, the bigger is the difference between 
    the two methods. The correlation between the bias, measured as the absolute 
    value between the difference of fractions of dark- and light-furred dogs, 
    and the difference in accuracy between the conditional and unconditional
    HSIC method is $0.858$.

    The test set accuracies
    reported here are lower than, for example, reported in
    \cite{kim2019learning}. This difference has two reasons. First, to guarantee
    a fair evaluation, we fixed the size of the training set to be the same for every
    bias level. This constraint reduces the number of images in the training set. Second, to solely focus on the bias in our training sets, 
    we refrain from pretraining on ImageNet \cite{deng2009imagenet} because this dataset already contains
    several thousand images of dogs. Instead, we train our networks from scratch, which leads to a 
    more objective evaluation for the debiasing methods.

\section{Conclusion}\label{sec:conclusion}
    In this work, we investigated a specific kind of bias described in
    Section~\ref{sec:bias}. The exact model formulation allowed us to provide 
    theoretical evidence, including a mathematical proof to confirm our proposed 
    solution. With these findings, our work clearly differs from most related 
    work on adversarial debiasing, which solely relies on empirical evaluations.

    Our experimental results also support our theoretical claims made in this 
    work. If a bias can be modeled as explained in Section~\ref{sec:bias}, a 
    conditional independence criterion is a better choice compared to an 
    unconditional independence criterion. This is the result of the theoretical 
    considerations in Section~\ref{sec:method} and confirmed by the
    experiments in Section~\ref{sec:experiments}, where the difference between 
    the two criteria has been maximized. We further demonstrated that this 
    difference is the reason for the increase in accuracy and that this increase 
    is also observable for real-world data.

    To estimate the effect in unbiased data or in situations where the bias 
    model does not apply, we used the unbiased situation (50\% in 
    Table~\ref{tab:real}) and the situations with only light- or dark-furred 
    dogs (0\% and 100\% in Table~\ref{tab:real}). In these experiment, our 
    method performs slightly worse but similar to the baseline.

\bibliographystyle{acm}
\bibliography{related_work}

\clearpage
\begin{appendices}
    \section{Detailed Steps to arrive at (11) and (12)} \label{sec:ape:calc}
    In this section we provide very detailed calculations on how to arrive at 
    equations (11) and (12).

    By definition of the partial correlation, we find
    \begin{equation}\label{eqn:first}
        \langle F^{\ast}(I), B \rangle | L = \langle F^{\ast}(I) - 
        \widehat{F^{\ast}_I}(L), B - \widehat{B}(L) \rangle.
    \end{equation}
    Here, $\widehat{F^{\ast}_I}(L)$ is the best linear regression of $F^{\ast}$ 
    given $L$ and $\widehat{B}(L)$ is the best linear regression of $B$ given $L$.

    Now we use the formula to calculate the linear regression coefficient and 
    write the best linear regression as the multiplication of the coefficient and 
    the value $L$ to get
    \begin{equation}
        \widehat{F^{\ast}_I}(L) = \frac{\langle F^{\ast}(I), L \rangle}
        {\langle L, L \rangle} L
    \end{equation}
    and 
    \begin{equation}
        \hat{B}(L) = \frac{ \langle B, L \rangle}{\langle L, L \rangle} L.
    \end{equation}
    We plug these results into (\ref{eqn:first}) to get
    \begin{equation}\label{eqn:paper11}
        \begin{split}
            \langle F^{\ast}(I)&, B \rangle | L = \\ 
            & \langle F^{\ast}(I) - \frac{\langle F^{\ast}(I), L \rangle} 
                {\langle L, L \rangle} L, B - \frac{ \langle B, L \rangle}{
                \langle L, L \rangle} L \rangle.
        \end{split}
    \end{equation}
    We expand the scalar product to find 
    \begin{equation}
        \begin{split}
            \langle &F^{\ast}(I), B \rangle | L = \\
                &\langle F^{\ast}(I), B \rangle 
            - \langle F^{\ast}(I), \frac{ \langle B, L \rangle}{\langle L, L 
                \rangle} L \rangle\\
            &- \langle \frac{\langle F^{\ast}(I), L \rangle} {\langle L, L 
                \rangle} L, B \rangle
            + \langle \frac{\langle F^{\ast}(I), L \rangle} {\langle L, L 
                \rangle} L, \frac{ \langle B, L \rangle}{\langle L, L \rangle} L 
                \rangle\\
            & = S_1 - S_2 - S_3 + S_4. 
        \end{split}
    \end{equation}

    Simplifying $S_2$ and $S_3$ these terms individually by pulling the fraction 
    out of the scalar product. We get 
    \begin{equation}
        S_2 = \frac{\langle F^{\ast}(I), L \rangle \langle B, L \rangle}{\langle
            L, L \rangle}
    \end{equation}
    and 
    \begin{equation}
        S_3 = \frac{ \langle B, L \rangle\langle F^{\ast}(I), L \rangle}{\langle
            L, L \rangle}.
    \end{equation}
    For $S_4$ we pull both fractions out of the scalar product and reduce the 
    resulting fraction by $\langle L, L \rangle$ to get
    \begin{equation}
        S_4 = \frac{\langle F^{\ast}(I), L \rangle \langle B, L \rangle}{\langle
            L, L \rangle^2} \langle L, L \rangle =  \frac{ \langle B, L \rangle
            \langle F^{\ast}(I), L \rangle}{\langle
            L, L \rangle}.
    \end{equation}
    As we find that
    \begin{equation}
        S_2 = S_3 = S_4
    \end{equation}
    we conclude that 
    \begin{equation}\label{eqn:10}
        \begin{split}
            \langle F^{\ast}(I), B \rangle | L &= S_1 - S_2 \\
            & = \langle F^{\ast}(I), B \rangle - \frac{\langle F^{\ast}(I), L 
                \rangle \langle B, L \rangle}{\langle L, L \rangle}.
        \end{split}
    \end{equation}

    Now we substitute $L$ by $\zeta_1S$ (definition in (6) of the paper), 
    $F^{\ast}(I)$ by $\zeta_1S$ (definition in (9) of the paper) and $B$ by
    $\alpha_1S + \alpha_2B^{\ast}$ (definition in (5) of the paper) to get:
    \begin{equation}
        \begin{split}
            \langle F^{\ast}(I), B \rangle &= \langle \zeta_1S, \alpha_1S + 
                \alpha_2B^{\ast} \rangle\\ 
            &= \zeta_1\alpha_1 \langle S, S \rangle + \zeta_1\alpha_2\langle S, 
                B^{\ast} \rangle.
        \end{split}
    \end{equation}
    Since $B^{\ast}$ and $S$ are independent by definition of $B^{\ast}$, we can
    evaluate $\langle S, B^{\ast} \rangle = 0$ and find
    \begin{equation}\label{eqn:12}
        \langle F^{\ast}(I), B \rangle = \zeta_1\alpha_1\langle S, S \rangle.
    \end{equation}
    Similar we find 
    \begin{equation}\label{eqn:paper121}
        \begin{split}
            &\frac{\langle F^{\ast}(I), L \rangle \langle B, L \rangle}{\langle 
                L, L \rangle}\\ 
            & \quad \quad = \frac{\langle\zeta_1S, \zeta_1S\rangle\langle 
                \alpha_1S + \alpha_2B^{\ast}, \zeta_1S \rangle}{\langle`
                \zeta_1S, \zeta_1S \rangle}\\ 
            & \quad \quad = \frac{\zeta_1^2\langle S, S \rangle(\alpha_1\zeta_1 
                \langle S, S \rangle + \alpha_2\zeta_1\langle B^{\ast}, S 
                \rangle)}{\zeta_1^2 \langle S,S \rangle}.
        \end{split}
    \end{equation}
    We use the fact that $\langle S, B^{\ast} \rangle = 0$ by definition of 
    $B^{\ast}$ to find 
    \begin{equation}
        \frac{\langle F^{\ast}(I), L \rangle \langle B, L \rangle}{\langle L, L 
        \rangle} = \frac{\zeta_1^2\langle S, S \rangle \alpha_1\zeta_1 \langle 
        S, S \rangle}{\zeta_1^2 \langle S,S \rangle}.
    \end{equation}
    Reducing this fraction by $\zeta_1^2\langle S, S \rangle$ leads to
    \begin{equation}\label{eqn:15}
        \frac{\langle F^{\ast}(I), L \rangle \langle B, L \rangle}{\langle L, L 
        \rangle} = \alpha_1\zeta_1 \langle S, S \rangle.
    \end{equation}
    Finally, we include the results of (\ref{eqn:12}) and (\ref{eqn:15}) into 
    (\ref{eqn:10}) to arrive at
    \begin{equation}
        \langle F^{\ast}(I), B \rangle | L = \alpha_1\zeta_1 \langle S, S 
        \rangle - \alpha_1\zeta_1 \langle S, S \rangle = 0.
    \end{equation}

    For the convenience of the reader we want to point out that 
    (\ref{eqn:paper11}) corresponds to (11) in the paper, (\ref{eqn:paper121})
    corresponds to the first line of (12). 

    \section{Evaluation with Methods of Algorithmic Fairness}
    \label{sec:ape:fair}
    During all runs of all methods we measured whether the classifier reached 
    ``equalized odds'' \cite{hardt2016equality} and ``demographic parity'' 
    \cite{dwork2012fairness}. Since this requires an independence test, we 
    report results tested with HSIC, RCOT and partial correlation (PC). The 
    level of significance is 0.01. The results of these tests for the experiment 
    on synthetic data can be seen in Table~\ref{tab:ape:eo}. The results for the
    experiments on real data can be found Table~\ref{tab:ape:eo_real}.

    \begin{table*}
        \caption{In this table we report how many out of the 100 runs resulted
        in a classifier that fulfills ``equalized odds''}
        \label{tab:ape:eo}
        \centering
        \begin{tabular}{lcccccc}
            \toprule
             & \multicolumn{3}{c}{Setup~I} &  \multicolumn{3}{c}{Setup~II}\\
             \cmidrule(lr){2-4}
             \cmidrule(lr){5-7}
            Method & HSIC & RCOT & PC & HSIC & RCOT & PC\\
            \midrule
            Baseline    & 0/100 & 0/100 & 0/100 & 100/100 & 99/100 & 100/100 \\
            Adeli et al. \cite{adeli2019bias} & 2/100 & 4/100 & 0/100 & 100/100 & 99/100 & 100/100 \\
            Zhang et al. I \cite{zhang2018mitigating} & 0/100 & 1/100 & 0/100 & 99/100 & 92/100 & 100/100 \\
            Zhang et al. II \cite{zhang2018mitigating}  & 1/100 & 1/100 & 0/100 & 100/100 & 97/100 & 100/100 \\
            Kim et al. \cite{kim2019learning}           & 2/100 & 0/100 & 0/100 & 0/100 & 0/100 & 100/100 \\
            \midrule
            Ours(MI) & 0/100 & 0/100 & 0/100 & 100/100 & 100/100 & 100/100 \\
            Ours(HSIC) & 0/100 & 0/100 & 0/100 & 100/100 & 99/100 & 100/100 \\
            Ours(MCC) & 0/100 & 0/100 & 0/100 & 99/100 & 100/100 & 100/100 \\
            \bottomrule
        \end{tabular}
    \end{table*}

\begin{table*}[th]
    \caption{Experimental results on the cats and dogs dataset. All methods were trained
    on a dataset in which p\% of all dogs are dark-furred dogs and $p$\% of all cats are light-furred.
    The first column of the table indicates the fraction $p$. The following columns
    contain the accuracies on an unbiased test set averaged over three runs. 
    We also report the standard error}
    \label{tab:ape:eo_real}

    \centering
    \begin{tabular}{ccccccc}
        \toprule
        Fraction & Baseline& Adeli et al. \cite{adeli2019bias}  & 
            Zhang et al. I \cite{zhang2018mitigating} & Zhang et al. II \cite{zhang2018mitigating} & Ours(HSIC) & HSIC -- uncon \\
        \midrule
        \multicolumn{7}{c}{HSIC}\\
        \cmidrule(lr){1-7}
        0\%     & 0/3 & 0/3 & 0/3 & 0/3 & 0/3 & 0/3 \\
        10\%    & 0/3 & 0/3 & 0/3 & 0/3 & 0/3 & 0/3 \\        
        20\%    & 0/3 & 0/3 & 0/3 & 0/3 & 0/3 & 0/3 \\       
        30\%    & 0/3 & 0/3 & 0/3 & 0/3 & 1/3 & 0/3 \\
        40\%    & 3/3 & 3/3 & 3/3 & 3/3 & 3/3 & 3/3 \\
        50\%    & 3/3 & 1/3 & 3/3 & 3/3 & 2/3 & 2/3 \\
        60\%    & 1/3 & 0/3 & 1/3 & 0/3 & 0/3 & 0/3 \\
        70\%    & 0/3 & 0/3 & 0/3 & 0/3 & 0/3 & 0/3 \\
        80\%    & 0/3 & 0/3 & 0/3 & 0/3 & 0/3 & 0/3 \\
        90\%    & 0/3 & 0/3 & 0/3 & 0/3 & 0/3 & 0/3 \\
        100\%   & 0/3 & 0/3 & 0/3 & 0/3 & 0/3 & 0/3 \\
        \\
        \multicolumn{7}{c}{RCOT}\\
        \cmidrule(lr){1-7}
        0\%     & 0/3 & 0/3 & 0/3 & 0/3 & 0/3 & 0/3 \\
        10\%    & 0/3 & 0/3 & 0/3 & 0/3 & 0/3 & 0/3 \\        
        20\%    & 0/3 & 0/3 & 0/3 & 0/3 & 0/3 & 0/3 \\       
        30\%    & 0/3 & 0/3 & 0/3 & 0/3 & 0/3 & 0/3 \\
        40\%    & 0/3 & 1/3 & 0/3 & 0/3 & 0/3 & 1/3 \\
        50\%    & 3/3 & 3/3 & 3/3 & 3/3 & 3/3 & 3/3 \\
        60\%    & 3/3 & 3/3 & 3/3 & 3/3 & 3/3 & 3/3 \\
        70\%    & 0/3 & 2/3 & 2/3 & 3/3 & 2/3 & 2/3 \\
        80\%    & 0/3 & 0/3 & 0/3 & 0/3 & 0/3 & 0/3 \\
        90\%    & 0/3 & 0/3 & 0/3 & 0/3 & 0/3 & 0/3 \\
        100\%   & 0/3 & 0/3 & 0/3 & 0/3 & 0/3 & 0/3 \\
        \bottomrule
    \end{tabular}
\end{table*}

    \section{List of Hyperparameters for different Methods}
    \label{sec:ape:hyp}
    In this section, we report the hyperparameters of the methods used. The 
    parameters were optimized via grid-search. For all methods, we report the
    learning rate for the classifier ($lr_c$) and the number of epochs (Nr. 
    Ep.). For all methods other than the baseline, we report $\beta$, the 
    weight of the adversarial debiasing loss. If a second neural network is 
    trained to calculate this loss, we report the learning rate for this network
    by $lr_b$. We further report whether we forced the mini batches to be 
    balanced (Bal.).

    \begin{table*}
        \caption{The hyperparameters for all methods and all experiments}
        \label{tab:ape:hp}
        \centering
        \begin{tabular}{lccccc}
            Method & $lr_c$ & Nr. Ep. & $\beta$ & $lr_b$ & Bal.\\
            \midrule
            & \multicolumn{5}{c}{Setup~I} \\ 
            \cmidrule(lr){2-6}
            Baseline    & $3$e$-5$ & 30     & --    & -- & False \\
            Adeli et al. \cite{adeli2019bias}    
                        & $3$e$-4$ & 1000   & 1    & $3$e$-4$ & False \\
            Zhang et al. I \cite{zhang2018mitigating}
                        & $3$e$-3$ & 30     & 2     & $3$e$-3$ & False \\
            Zhang et al. II \cite{zhang2018mitigating}
                        & $3$e$-3$ & 30     & 0.5   & $3$e$-3$ & False \\
            Kim et al. \cite{kim2019learning}
                        & $3$e$-4$ & 1000   & 1     & $3$e$-4$ & False \\
            Ours(MI)
                        & $1$e$-5$ & 100    & 0.0625& $1$e$-5$ & False \\
            Ours(MCC)
                        & $3$e$-5$ & 30     & 0.0625& $3$e$-5$ & False \\
            Ours(HSIC)
                        & $3$e$-5$ & 30     & 0.0625& $3$e$-5$ & True \\
            Unconditional HSIC
                        & $3$e$-5$ & 30     & 0.003 & $3$e$-5$ & True \\
            Unconditional MI
                        & $1$e$-5$ & 100    & 0.0625& $1$e$-5$ & False \\
            Ours(MCC) -- only MCC
                        & $3$e$-4$ & 1000    & 1 & $3$e$-4$ & False \\
            Ours(MCC) -- only PC
                        & $3$e$-5$ & 30    & 0.0625  & $3$e$-5$ & False \\
            \\
            & \multicolumn{5}{c}{Setup~II} \\ 
            \cmidrule(lr){2-6}
             Baseline   & $3$e$-5$ & 100    & --    & -- & False \\
            Adeli et al. \cite{adeli2019bias}    
                        & $3$e$-5$ & 100    & 1    & $3$e$-5$ & False \\
            Zhang et al. I \cite{zhang2018mitigating}
                        & $3$e$-3$ & 30     & 0.5   & $3$e$-3$ & False \\
            Zhang et al. II \cite{zhang2018mitigating}
                        & $3$e$-3$ & 30     & 0.5   & $3$e$-3$ & False \\
            Kim et al. \cite{kim2019learning}
                        & $3$e$-5$ & 100    & 1     & $3$e$-5$ & False \\
            Ours(MI)
                        & $1$e$-5$ & 100    & 0.05  & $1$e$-5$ & False \\
            Ours(MCC)
                        & $3$e$-5$ & 30     & 0.0625  & $3$e$-5$ & False \\
            Ours(HSIC)
                        & $3$e$-5$ & 30     & 0.0625& $3$e$-5$ & True \\
            Unconditional HSIC
                        & $3$e$-5$ & 30     & 0.0625& $3$e$-5$ & True \\
            Unconditional MI
                        & $1$e$-5$ & 100    & 0.05& $1$e$-5$ & False \\
            Ours(MCC) -- only MCC
                        & $3$e$-5$ & 30    & 0.0625  & $3$e$-5$ & False \\
            Ours(MCC) -- only PC
                        & $3$e$-5$ & 30    & 0.0625  & $3$e$-5$ & False \\
                        \\
            & \multicolumn{5}{c}{Real-World Data} \\
            \cmidrule(lr){2-6}
            Baseline    & $1$e$-2$ & 150    & --    & -- & False \\
            Adeli et al. \cite{adeli2019bias}    
                        & $1$e$-2$ & 150    & 1    & $3$e$-4$ & False \\
            Zhang et al. I \cite{zhang2018mitigating}
                        & $1$e$-2$ & 150    & 1     & $3$e$-5$ & False \\
            Zhang et al. II \cite{zhang2018mitigating}
                        & $1$e$-2$ & 150    & 1     & $3$e$-5$ & False \\
            Ours(HSIC)
                        & $1$e$-2$ & 150    & 1     & $1$e$-5$ & False\\
            Unconditional HSIC
                        & $1$e$-2$ & 150    & 1     & $1$e$-5$ & False\\
            \bottomrule
        \end{tabular}
    \end{table*}

\end{appendices}
\end{document}